\documentclass[a4paper]{article}
\usepackage[all]{xy}\usepackage[latin1]{inputenc}        
\usepackage[dvips]{graphics,graphicx}
\usepackage{amsfonts,amssymb,amsmath,color,mathrsfs, amstext}
\usepackage{amsbsy, amsopn, amscd, amsxtra, amsthm,authblk}
\usepackage{enumerate,algorithmicx,algorithm}
\usepackage{algpseudocode}
\usepackage{upref}
\usepackage{geometry}
\usepackage[displaymath]{lineno}
\usepackage{float}
\usepackage{yhmath}
\usepackage{booktabs}
\usepackage{subcaption}
\usepackage{multirow}
\usepackage{makecell}
\usepackage[colorlinks,
            linkcolor=red,
            anchorcolor=red,
            citecolor=red
            ]{hyperref}

\numberwithin{equation}{section}

\DeclareMathOperator{\kl}{KL}

\newcommand{\bbR}{\mathbb{R}}

\newtheorem{theorem}{Theorem}
\newtheorem{lemma}{Lemma}

\newtheorem{remark}{Remark}

\numberwithin{equation}{section}

\begin{document}

\title{A stochastic version of Stein variational gradient descent
for efficient sampling}

\author[1]{Lei Li\thanks{leili2010@sjtu.edu.cn}}
\author[2]{Yingzhou Li\thanks{yingzhou.li@duke.edu}}
\author[2,3]{Jian-Guo Liu\thanks{jliu@phy.duke.edu}}
\author[2]{Zibu Liu\thanks{zibu.liu@duke.edu}}
\author[2,3,4]{Jianfeng Lu\thanks{jianfeng@math.duke.edu}}

\affil[1]{School of Mathematical Sciences, Institute of Natural
Sciences, MOE-LSC, Shanghai Jiao Tong University, Shanghai, 200240, P.
R. China.}
\affil[2]{Department of Mathematics, Duke University, Durham, NC 27708,
USA.}
\affil[3]{Department of Physics, Duke University, Durham, NC 27708,
USA.}
\affil[4]{Department of Chemistry, Duke University, Durham, NC 27708,
USA.}

\date{}
\maketitle

\begin{abstract}
We propose in this work RBM-SVGD, a stochastic version of Stein
Variational Gradient Descent (SVGD) method for efficiently sampling
from a given probability measure and thus useful for Bayesian
inference. The method is to apply the Random Batch Method (RBM) for
interacting particle systems proposed by Jin et al to the interacting
particle systems in SVGD. While keeping the behaviors of SVGD,
it reduces the computational cost, especially when the interacting
kernel has long range. We prove that the one marginal distribution of
the particles generated by this method converges to the one marginal
of the interacting particle systems under Wasserstein-2 distance on
fixed time interval $[0, T]$. Numerical examples verify the efficiency
of this new version of SVGD.
\end{abstract}

\section{Introduction}

The empirical measure with samples from some probability measure (which
might be known up to a multiplicative factor) has many applications
in Bayesian inference \cite{box2011bayesian,blei2017variational} and
data assimilation \cite{MR3363508}.  A class of widely used sampling
methods is the Markov Chain Monte Carlo (MCMC) methods, where the
trajectory of a particle is given by some constructed Markov chain with
the desired distribution invariant. The trajectory of the particle is
clearly stochastic, and the Monte Carlo methods take effect slowly
for small number of samples.  Unlike MCMC, the Stein variational
Gradient method (proposed by Liu and Wang in \cite{liu2016stein})
belongs to particle based variational inference sampling methods (see
also \cite{rezende2015variational,dai2016provable}). These methods
update particles by solving optimization problems, and each iteration
is expected to make progress.  As a non-parametric variational
inference method, SVGD gives a deterministic way to generate points
that approximate the desired probability distribution by solving
an ODE system.  Suppose that we are interested in some target
probability distribution with density $\pi(x)\propto \exp(-V(x))$
($x\in\mathbb{R}^d$). In SVGD, one sets $V=-\log \pi$ and solve
the following ODE system for given initial points $\{X_i(0)\}$
(see \cite{liu2016stein,liu2017stein}):
\begin{gather}\label{eq:discreteODE}
\dot{X}_i=
\frac{1}{N}\sum_{j=1}^N \nabla_y \mathcal{K}(X_i,X_j)
-\frac{1}{N}\sum_{j=1}^N \mathcal{K}(X_i,X_j)\nabla V(X_j),
\end{gather}
where $\mathcal{K}(x, y)$ is a symmetric positive definite kernel.  When
$t$ is large enough, the empirical measures constructed using
$\{X_i(t)\}$ is expected to be close to $\pi$.

SVGD seems to be more efficient in the particle level for approximating
the desired measure and interestingly, it reduces to the maximum a
posterior (MAP) method when $N=1$ \cite{liu2016stein}.  It provides
consistent estimation for generic distributions as Monte Carlo methods
do, but with fewer samples.  Theoretic understanding of
\eqref{eq:discreteODE} is limited.  For example, the convergence of the
particle system \eqref{eq:discreteODE} is still open. Recently, there
are a few attempts for the understanding of the limiting mean field PDE
\cite{liu2017stein,lu2018scaling}. In particular, Lu et al
\cite{lu2018scaling} showed the convergence of the PDE to the desired
measure.

Though \eqref{eq:discreteODE} behaves well when the particle number $N$
is not very big, one sometimes still needs efficient algorithm to
simulate \eqref{eq:discreteODE}. For example, in a typical MCMC method
$N=10^4\sim 10^6$ while in SVGD, one may have $N=10^2\sim 10^3$. Though
$N=10^2\sim 10^3$ is not large, simulating \eqref{eq:discreteODE} needs
$O(N^2)$ work to compute the interactions for each iteration,
especially for interaction kernels that are not super localized (such as
kernels with algebraic decaying rate, like $K(x)\sim |x|^{-\alpha}$).
The computation cost of SVGD for these cases is therefore comparable
with MCMC with larger number of particles. Hence, it is highly motivated
to develop a cheap version of SVGD.

In this work, we propose RBM-SVGD, a stochastic version of SVGD for
sampling from a given probability measure. The idea is very natural: we
apply the random batch method in \cite{jin2018random} to the interacting
particle system \eqref{eq:discreteODE}. Note that in the random batch
method, the 'batch' refers to the set for computing the interaction
forces, not to be confused with the 'batch' of samples for computing
gradient as in stochastic gradient descent (SGD). Of course, if $V$ is
the loss function corresponding to many samples, or the probability
density in Bayesian inference corresponding to many observed data, the
data-mini-batch idea can be used to compute $\nabla V$ in SVGD as well (see
\cite{liu2016stein}).  With the random batch idea for computing
interaction, the complexity for each iteration now is only $O(N)$.
Moreover, it inherits the advantages of SVGD (i.e. efficient for
sampling when the number of particles is not large) since the random
batch method is designed to approximate the particle system directly. In
fact, we will prove that the one marginal of the random batch method
converges to the one marginal of the interacting particle systems under
Wasserstein-2 distance on fixed time interval $[0, T]$.  Note that the
behavior of randomness in RBM-SVGD is different from that in MCMC. In
MCMC, the randomness is required to ensure that the desired probability
is invariant under the transition. The randomness in RBM-SVGD is simply
due to the batch for computing the interaction forces, which is mainly
for speeding the computation. Though this randomness is not essential
for sampling from the invariant measure, it may have other benefits. For
example, it may lead to better ergodic properties for particle system.

\section{Mathematical background of SVGD}\label{sec:svgd}

We now give a brief introduction to the SVGD proposed in
\cite{liu2016stein} and make some discussion. The derivation here is a
continuous counterpart of that in \cite{liu2016stein}.

Assume that random variable $X\in \mathbb{R}^d$ has density $p_0(x)$.
Consider some mapping $T: \mathbb{R}^d\to\mathbb{R}^d$ and we denote the
distribution of $T(X)$ by $p:=T_{\#}p_0$, which is called the
push-forward of $p_0$ under $T$. The goal is to make $T_{\#}p_0$ closer
to $\pi(x)$ in some sense. The way to measure the closeness of measures
in  \cite{liu2016stein} is taken to be the Kullback-Leibler (KL)
divergence, which is also known as the relative entropy, defined by
\begin{gather}
    \kl(\mu || \nu)=\mathbb{E}_{Y\sim
    \mu}\log\left(\frac{d\mu}{d\nu}(Y)\right),
\end{gather}
where $\frac{d\mu}{d\nu}$ is the well-known Radon-Nikodym derivative. In
\cite[Theorem 3.1]{liu2016stein}, it is shown that the Frechet
differential of $T\mapsto G(T):=\kl(p || \pi)$ is given by
\begin{gather}
    \langle \frac{\delta G}{\delta T}, \phi\rangle = -\mathbb{E}_{Y\sim
    p} S_{\pi}\phi(Y),~~\forall \phi\in C_c^{\infty}(\mathbb{R}^d;
    \bbR^d)
\end{gather}
where $S_{q}$ associated with a probability density $q$ is called the
Stein operator given by
\begin{gather}
    S_{q}\phi(x)=\nabla(\log q(x))\cdot\phi(x)+\nabla\cdot\phi(x).
\end{gather}
In fact, using the formula
\begin{gather}\label{eq:effectofmapping}
    \frac{d}{d\epsilon}(T+\epsilon \phi\circ T)_{\#}p_0
    |_{\epsilon=0}=\frac{d}{d\epsilon}(I+\epsilon
    \phi)_{\#}p|_{\epsilon=0}= -p S_p\phi=-\nabla\cdot(p \phi),
\end{gather}
and $\frac{\delta \kl(p || \pi)}{\delta p}=\log p-\log\pi$, one finds
\begin{gather}
    \langle \frac{\delta G}{\delta T}, \phi\rangle =\left\langle
    \frac{\delta \kl(p || \pi)}{\delta p}, -\nabla\cdot(p \phi)
    \right\rangle =-\int_{\mathbb{R}^d}  p S_{\pi}\phi\,dx.
\end{gather}
The quantity $\langle \frac{\delta G}{\delta T}, \phi\rangle $ can be
understood as the directional derivative of $G(\cdot)$ in the direction
given by $\phi$.

Based on this calculation, we now consider a continuously varying family
of mappings $T_{\tau}, \tau\ge 0$ and
\[
    \frac{d}{d\tau}{T}_{\tau}=\phi_{\tau}\circ T_{\tau}.
\]
Here, '$\circ$' means composition, i.e. for any given $x$,
$\frac{d}{d\tau}T_{\tau}(x)=\phi_{\tau}(T_{\tau}(x))$.  In this sense
$x\mapsto X(\tau; x):=T_{\tau}(x)$ is the trajectory of $x$ under this
mapping; $x$ can be viewed as the so-called Lagrangian coordinate as in
fluid mechanics while $\phi_{\tau}$ is the flow field. We denote
\begin{gather}
    p_{\tau}:=(T_{\tau})_{\#}p_0.
\end{gather}
The idea is then to choose $\phi_{\tau}$ such that the functional $\tau
\mapsto G(T_{\tau})$ decays as fast as possible. Note that to optimize
the direction, we must impose the field to have bounded magnitude
$\|\phi_{\tau}\|_H\le 1$, where $H$ is some subspace of the functions
defined on $\mathbb{R}^d$. The optimized curve $\tau\mapsto T_{\tau}$ is
a constant speed curve (in some manifold).  Hence, the problem is
reduced to the following optimization problem
\begin{gather}\label{eq:opt}
    \sup\{ \mathbb{E}_{Y\sim p} S_{\pi}\phi(Y) | \|\phi\|_H\le 1 \}.
\end{gather}

It is observed in \cite{liu2016stein} that this optimization problem can
be solved by a convenient closed formula if $H$ is the so-called
(vector) reproducing kernel Hilbert space (RKHS)
\cite{aronszajn1950theory,berlinet2011reproducing}.  A (scalar) RKHS is
a Hilbert space, denoted by $\mathcal{H}$, consisting of functions
defined on some space $\Omega$ (in our case $\Omega=\mathbb{R}^d$) such
that the evaluation function $f\mapsto E_x(f):=f(x)$ is continuous for
all $x\in\Omega$. There thus exists $k_x\in \mathcal{H}$ such that
$E_x(f)=\langle f, k_x\rangle_\mathcal{H}$. Then the kernel
$\mathcal{K}(x, y):=\langle k_x, k_y\rangle_\mathcal{H}$ is symmetric
and positive definite, meaning that  $\sum_{i=1}^n\sum_{j=1}^n
\mathcal{K}(x_i, x_j)c_i c_j\ge 0$ for any $x_i\in\Omega$ and $c_i\in
\mathbb{R}$. Reversely, given any positive definite kernel, one can
construct a RKHS consisting of functions $f(x)$ of the form $f(x)=\int
\mathcal{K}(x, y)\psi(y)\,d\mu(y)$ where $\mu$ is some suitably given
measure on $\Omega$. For example, if $\mu$ is the counting measure,
choosing $\psi(y)=\sum_{j=1}^{\infty}a_j 1_{x_j}(y)$
($a_j\in\mathbb{R}$) can recover the form of RKHS in
\cite{liu2016stein}. All such constructions yield isomorphic RKHS as
guaranteed by Moore-Aronszajn theorem \cite{aronszajn1950theory}.  Now,
consider a given $\mu$ and $H=\mathcal{H}^d$ to be the vector RKHS:
\[
    H=\left\{f=\int_{\mathbb{R}^d} \mathcal{K}(\cdot, y)\psi(y)\,d\mu(y)
    \Big| \psi: \mathbb{R}^d\to \mathbb{R}^d,
    \iint_{\mathbb{R}^d\times\mathbb{R}^d} \mathcal{K}(x, y)\psi(x)\cdot
    \psi(y) d\mu(x)d\mu(y)<\infty \right\}.
\]
The inner product is defined as 
\begin{gather}
    \begin{split}
        \langle f^{(1)}, f^{(2)}\rangle_{H}
        &=\iint_{\bbR^d\times\bbR^d} \mathcal{K}(x,
        y)\psi^{(1)}(x)\cdot \psi^{(2)}(y)\,d\mu(x) d\mu(y) \\
        &=\sum_{j=1}^d \iint_{\bbR^d\times\bbR^d} \mathcal{K}(x,
        y)\psi_j^{(1)}(x)\psi^{(2)}_j(y)\,d\mu(x) d\mu(y).
    \end{split}
\end{gather}
This inner product therefore induces a norm $\|f\|_H=\sqrt{\langle f,
f\rangle_H }$. Clearly, $H$ consists of functions with $\|\cdot \|_H$ to
be finite.  The optimization problem \eqref{eq:opt} can be solved by the
Lagrange multiplier method
\[
    \mathcal{L}=\int_{\bbR^d} (S_{\pi}\phi)
    p_{\tau}(y)\,dy-\lambda\iint_{\bbR^d\times\bbR^d} \mathcal{K}(x,
    y)\psi(x)\cdot \psi(y)\,d\mu(x) d\mu(y),
\]
where $dy$ means Lebesgue measure and $\phi(x)=\int_{\bbR^d}
\mathcal{K}(x, y)\psi(y)\,d\mu(y)$. Using $\frac{\delta
\mathcal{L}}{\delta\phi}=0$, we find
\begin{gather}
    2\lambda \phi=\int_{\bbR^d} \mathcal{K}(x,
    y)(S_{\pi}^*p_t)(y)\,dy=:\mathcal{V}(p_t),
\end{gather}
where $S_{\pi}^*$ is given by
\begin{gather}
    S_{\pi}^*(f)=f(y)\nabla(\log\pi)-\nabla f(y)=-f(y)\nabla V(y)-\nabla
    f(y).
\end{gather}
The ODE flow
\[
    \frac{d}{d\tau}T_\tau=\frac{1}{2\lambda(\tau)}\mathcal{V}(p_\tau)\circ
    T_\tau,
\]
gives the constant speed optimal curve, so that the velocity is the unit
vector in $H$ along the gradient of $G$. Re-parametrizing the curve
$t=t(\tau)$ so that $\frac{d \tau}{d t}=2\lambda$,  and we denote
$\rho_{t}:=p_{\tau(t)}$, then
\begin{gather}\label{eq:odeflow}
    \frac{d}{dt}T_t=\mathcal{V}(\rho_t)\circ T_t.
\end{gather}
Clearly, the curve of $T_t$ is not changed by this reparametrization.
Using \eqref{eq:effectofmapping}, one finds that $\rho$ satisfies the
following equation
\begin{gather}\label{eq:meanfield}
    \partial_t \rho=-\nabla\cdot(\mathcal{V}(\rho)
    \rho)=\nabla\cdot(\rho \mathcal{K}*(\rho\nabla V+\nabla \rho)).
\end{gather}
Here, $\mathcal{K}*f(x):=\int\mathcal{K}(x,y)f(y)dy$. It is easy to see
that $\exp(-V)$ is invariant under this PDE.  According to the
explanation here, the right hand side gives the optimal decreasing
direction of KL divergence if the transport flow is measured by RKHS.
Hence, one expects it to be the negation of gradient of KL divergence in
the manifold of probability densities with metric defined through RKHS.
Indeed, Liu made the first attempt to justify this in \cite[Sec.
3.4]{liu2017stein}.

While everything looks great for continuous probability densities, the
above theory does not work for empirical measures because the KL
divergence is simply infinity. For empirical measure, $\nabla \rho$ must
be in the distributional sense. However, the good thing for RKHS is that
we can move the gradient from $\nabla\rho$ onto the kernel
$\mathcal{K}(x,y)$ so that the flow \eqref{eq:odeflow} becomes
\eqref{eq:discreteODE}, which makes perfect sense. In fact, if
\eqref{eq:discreteODE}, holds, the empirical measure is a measure
solution to \eqref{eq:meanfield} (by testing on smooth function
$\varphi$) \cite[Proposition 2.5]{lu2018scaling}.  Hence, one expects
that \eqref{eq:discreteODE} will give approximation for the desired
density. The numerical tests in \cite{liu2016stein} indeed justify this
expectation. In this sense, the ODE system is formally a gradient flow
of KL divergence, though the KL divergence functional is infinity for
empirical measures.

Typical examples of $\mathcal{K}(x, y)$ include $\mathcal{K}(x,
y)=(\alpha x\cdot y+1)^{m}$, Gaussian kernel $\mathcal{K}(x,
y)=e^{-|x-y|^2/(2\sigma^2)}$ for $\mathbb{R}^d$, and $\mathcal{K}(x,
y)=\frac{\sin a(x-y)}{\pi(x-y)}$ for 1D space $\mathbb{R}$.  By
Bochner's theorem \cite{rudin2017}, if a function $K$ has a positive
Fourier transform, then
\begin{gather}
    \mathcal{K}(x,y)=K(x-y)
\end{gather}
is a positive definite kernel. With this kernel, \eqref{eq:discreteODE}
becomes
\begin{gather}\label{eq:discreteODE1}
    \dot{X}_i= -\frac{1}{N}\sum_{j=1}^N \nabla K(X_i-X_j)
    -\frac{1}{N}\sum_{j=1}^N K(X_i-X_j)\nabla V(X_j),
\end{gather}
as used in \cite{lu2018scaling}.  Both Gaussians and $1/|x|^{\alpha}$
with $\alpha\in (0, d)$ have positive Fourier transforms. The difference
is that Gaussian has short range of interaction while the latter has
long range of interaction.  One can smoothen $1/|x|^{\alpha}$ out by
mollifying with Gaussian kernels, resulting in positive definite smooth
kernels but with long range interaction. Choosing localized kernels like
Gaussians may have some issues in very high dimensional spaces
\cite{francois2005locality,detommaso2018stein}. Due to its simplicity,
when the dimension is not very high, we choose Gaussian kernels in
section \ref{sec:experiment}.

As a further comment, one may consider other metric to gauge the
closeness of probability measures, such as Wasserstein distances. Also,
one can consider other norms for $\phi$ and get gradient flows in
different spaces. These variants have been explored by some authors
already \cite{liu2018riemannian,chen2018unified}. In general, computing
the Frechet derivatives in closed form for these variants seems not that
easy.

\begin{remark}
    If we optimize \eqref{eq:opt} for $\phi$ in $L^2(\mathbb{R}^d;
    \mathbb{R}^d)$ spaces, the flow is then given by
    \begin{gather}
        \frac{d}{dt}T= (S_{\pi}^* \rho)\circ T.
    \end{gather}
    The corresponding PDE is $\partial_t\rho=\nabla\cdot(\rho
    (\rho\nabla V+\nabla \rho))=\nabla\cdot(\rho^2\nabla \log
    \frac{\rho}{\pi})$. This is in fact the case when we choose
    $\mathcal{K}(x,y)=\delta(x-y)$. This PDE, however, will not make
    sense for empirical measures since $\rho\nabla\rho$ is hard to
    justify (Clearly, the equivalent ODE system has the same trouble.)
    By using RKHS, the derivative on $\nabla\rho$ can be moved onto the
    kernel and then the ODE system makes sense.
\end{remark}

\section{The new sampling algorithm: RBM-SVGD}

We consider in general the particle system of the following form.
\begin{gather}\label{eq:generalsys}
    \dot{X}_i = \frac{1}{N}\sum_{j=1}^N
    F(X_i,X_j)=\frac{1}{N}F(X_i,X_i)+\frac{1}{N}\sum_{j: j\neq i}F(X_i,
    X_j).
\end{gather}
Here, $F(x, y)$ does not have to be symmetric, and also $F(x, x)$ is not
necessarily zero.

\subsection{The algorithms}

We apply the random batch method in \cite{jin2018random} to this
particle system.  In particular, choose a time step $\eta$. We define
time grid points
\begin{gather}
    t_m= m\eta.
\end{gather}
The idea of random batch method is to form some random batches at $t_m$
and then turn on interactions inside batches only. As indicated in
\cite{jin2018random}, the random division of the particles into $n$
batches takes $O(N)$ operations (we can for example use random
permutation). Depending on whether we do batches without or with
replacement, we can have different versions (see Algorithm
\ref{randomdiv} and \ref{fullyrandom}). For the ODEs in the algorithms,
one can apply any suitable ODE solver. For example, one can use the
forward Euler discretization if $F$ is smooth like Gaussian kernels.  If
$K$ is singular, one may take $p=2$ and apply the splitting strategy in
\cite{jin2018random}.

\begin{algorithm}[H]
    \caption{(Random Batch Method without replacement)}\label{randomdiv}
    \begin{algorithmic}[1]
        \For {$m \text{ in } 1: N_T$}   
            \State Divide $\{1, 2, \ldots, pn\}$ into $n$ batches randomly.
            \For {each batch  $\mathcal{C}_q$} 
                \State Update $X_i$'s ($i\in \mathcal{C}_q$) by
                solving the equation for $t\in [t_{m-1}, t_m)$.
                \begin{gather}\label{eq:firstalgorithm}
                    \dot{X}_i = \frac{1}{N} F(X_i, X_i) + (1 -
                    \frac{1}{N}) \frac{1}{p-1} \sum_{j\in
                    \mathcal{C}_q,j\neq i} F(X_i, X_j).
                \end{gather}
            \EndFor
        \EndFor
    \end{algorithmic}
\end{algorithm}

\begin{algorithm}[H]
    \caption{(Random Batch Method with replacement)}\label{fullyrandom}
    \begin{algorithmic}[1]
        \For {$m \text{ in } 1: N_T*(N/p)$}   
            \State Pick a set $\mathcal{C}$ of size $p$ randomly.
            \State Update $X^i$'s ($i\in \mathcal{C}$) by solving
            the following with pseudo-time $s\in [s_{m-1}, s_m)$.
            \begin{gather}\label{eq:secondalgB}
                \dot{X}_i = \frac{1}{N} F(X_i, X_i) + (1 -
                \frac{1}{N}) \frac{1}{p-1} \sum_{ j \in
                \mathcal{C},j\neq i}F(X_i, X_j).
            \end{gather}
         \EndFor
    \end{algorithmic}
\end{algorithm}

For the Stein Variational Gradient Descent \eqref{eq:discreteODE}, the
kernel $F$ takes the following form.
\begin{gather}\label{eq:forcesvgd}
    F(x, y)=\nabla_y \mathcal{K}(x,y)-\mathcal{K}(x,y)\nabla V(y).
\end{gather}
Applying the random batch method to this special kernel and using any
suitable ODE solvers, we get a class of sampling algorithms, which we
will call RBM-SVGD.  In this work, we will mainly focus on the ones
without replacement. Some discussion for RBM-SVGD with or without
replacement will be made in section \ref{sec:regression}. The one with
forward Euler discretization (with possible variant step size) is shown
in Algorithm \ref{ssvgd}. Clearly, the complexity is $O(pN)$ for each
iteration.

\begin{algorithm}[H]
    \caption{RBM-SVGD}\label{ssvgd}
    \begin{algorithmic}[1]
        \For {$k \text{ in } 0: N_T-1$}   
            \State Divide $\{1, 2, \ldots, pn\}$ into $n$ batches randomly.
            \For {each batch  $\mathcal{C}_q$} 
                \State For all $i\in \mathcal{C}_q$,
                \[
                    X_i^{(k+1)}\leftarrow X_i^{(k)}
                    + \frac{1}{N} \Big( \nabla_y
                    \mathcal{K}(X_i^{(k)}, X_i^{(k)})
                    - \mathcal{K}(X_i^{(k)}, X_i^{(k)})
                    \nabla V(X_i^{(k)})\Big)\eta_{k}
                    +\Phi_{k,i}\eta_k,
                \]
                where
                \begin{gather}
                    \Phi_{k,i} = \frac{N-1}{N(p-1)}
                    \sum_{j\in \mathcal{C}_q,j\neq i}
                    \left(\nabla_y \mathcal{K}(X_i^{(k)},X_j^{(k)})
                    -\mathcal{K}(X_i^{(k)}, X_j^{(k)})
                    \nabla V(X_j^{(k)})\right).
                \end{gather}
            \EndFor
        \EndFor
    \end{algorithmic}
\end{algorithm}
Here, $N_T$ is the number of iterations and $\{\eta_k\}$ is the sequence
of time steps, which play the same role as learning rate in stochastic
gradient descent (SGD). For some applications,  one may simply set $\eta_k=\eta\ll 1$ 
to be a constant and gets reasonably good results. However, in many high dimensional problems,
choosing $\eta_k$ to be constant may yield divergent sequences \cite{robbins1951stochastic}. One may decreas $\eta_k$ to obtain convergent data sequences. For example, one may simply choose $\eta_k=\eta\ll 1$ as in SGD. Another frequently used strategy is the Adagrad approach \cite{duchi2011adaptive,ward2018adagrad}.

\subsection{Theoretic results}

We now give convergence analysis regarding the time continuous version
of RBM-SVGD on torus $\mathbb{T}^d$ (i.e. choosing the particular force
\eqref{eq:forcesvgd} for Algorithm \ref{randomdiv} and
$X_i\in\mathbb{T}^d$). The derivation of SVGD clearly stays unchanged
for torus.  The reason we consider torus is that \eqref{eq:discreteODE}
is challenging to analyze in $\mathbb{R}^d$ because of the nonlocal
effect of the external force. On torus, all functions are smooth and
bounded. Moreover, using bounded domains with periodic boundary
condition can always approximate the problem in $\mathbb{R}^d$ in
practice.

Consider the random force for $z=(x_1, \ldots, x_N)\in\mathbb{T}^{Nd}$
defined by
\begin{gather}
    f_i(z) :=\left(1-\frac{1}{N} \right)
    \frac{1}{p-1}\sum_{j: j\in \mathcal{C}}F(x_i, x_j),
\end{gather}
where $\mathcal{C}$ is the random batch that contains $i$ in the random
batch method.  Correspondingly, the exact force is given by
$F_i(z)=\frac{1}{N}\sum_{j: j\neq i}F(x_i, x_j)$.  Define the 'noise' by
\begin{gather}
    \chi_i(z):=\frac{1}{N}\sum_{j: j\neq i}F(x_i, x_j)-f_i(z).
\end{gather}

We have the following consistency result regarding the random batch.
\begin{lemma}\label{thm:consistency}
    For given $z=(x_1, \ldots, x_N)\in \mathbb{T}^{Nd}$ (or
    $\mathbb{R}^{Nd}$), it holds that
    \begin{gather}
        \mathbb{E}\chi_i(z)=0
    \end{gather}
    Moreover, the second moment is given by
    \begin{gather}
        \mathbb{E}|\chi_i(z)|^2 = (1-\frac{1}{N})^2
        \left(\frac{1}{p-1}-\frac{1}{N-1}\right)\Lambda_i(z),
    \end{gather}
    where
    \begin{gather}
        \Lambda_i(z)=\frac{1}{N-2}
        \sum_{j: j\neq i}\Big| F(x_i, x_j)-\frac{1}{N-1}
        \sum_{k: k\neq i}F(x_i, x_k)  \Big|^2.
    \end{gather}
\end{lemma}
The proof is similar as in \cite{jin2018random}, but we also attach it
in the Appendix \ref{app:proofconsistency} for convenience.

We recall that the Wasserstein-$2$ distance is given by
\cite{santambrogio2015}
\begin{gather}\label{eq:W2}
    W_2(\mu, \nu)=\left(\inf_{\gamma \in
    \Pi(\mu,\nu)}\int_{\mathbb{T}^d\times\mathbb{T}^d}|x-y|^2
    d\gamma\right)^{1/2},
\end{gather}
where $\Pi(\mu,\nu)$ is called the transport plan, consisting of all the
joint distributions whose marginal distributions are $\mu$ and $\nu$
respectively: i.e. for any Borel set $E\subset \mathbb{T}^d$,
$\mu(E)=\iint_{x\in E, y\in\mathbb{T}^d}\gamma(dx, dy)$ and
$\nu(E)=\int_{x\in \mathbb{T}^d, y\in E}\gamma(dx, dy)$.

We now state the convergence result for the time continuous version of
RBM-SVGD, where $F(x, y)=\nabla_y
\mathcal{K}(x,y)-\mathcal{K}(x,y)\nabla V(y)$.  We use $\tilde{X}$ to
denote the process generated by the random algorithm while $X$ is the
process by \eqref{eq:discreteODE}.
\begin{theorem}\label{thm:conv}
    Assume $V$ and $K$ are smooth on torus $\mathbb{T}^d$. The initial
    data $X_i^0$ are drawn independently from the same initial
    distribution.  Given $T>0$, there exists $C(T)>0$, such that
    $\mathbb{E}|X_i-\tilde{X}_i|^2\le C(T)\eta$.  Consequently, the one
    marginals $\mu_N^{(1)}$ and $\tilde{\mu}_N^{(1)}$ are close under
    Wasserstein-$2$ distance:
    \[
        W_2(\mu_N^{(1)}, \tilde{\mu}_N^{(1)})\le C(T)\sqrt{\eta}.
    \]
\end{theorem}

\begin{proof}
    In the proof below, the constant $C$ will represent a general
    constant independent of $N$ and $p$, but its concrete meaning can
    change for every occurrence.

    Consider the corresponding two processes and $t\in [t_{m-1}, t_m]$.
    \begin{multline}\label{eq:twoprocess}
        \frac{d}{dt}\tilde{X}_i = \frac{1}{N} \left(\nabla_y
        \mathcal{K}(\tilde{X}_i,\tilde{X}_i) - \mathcal{K}(\tilde{X}_i,
        \tilde{X}_i) \nabla V(\tilde{X}_i)\right)\\
        +\frac{1-1/N}{p-1}\sum_{j: j\in \mathcal{C}}(\nabla_y
        \mathcal{K}(\tilde{X}_i,\tilde{X}_j) -
        \mathcal{K}(\tilde{X}_i,\tilde{X}_j)\nabla V(\tilde{X}_j)).
    \end{multline}
    and
    \begin{multline}
        \frac{d}{dt}X_i = \frac{1}{N}\left(\nabla_y
        \mathcal{K}(X_i,X_i)-\mathcal{K}(X_i, X_i)
        \nabla V(X_i)\right) \\
        +\frac{1}{N}\sum_{j: j\neq i}(\nabla_y
        \mathcal{K}(X_i,X_j)-\mathcal{K}(X_i,X_j)\nabla V(X_j)).
    \end{multline}

    Taking the difference and dotting with $\tilde{X}_i-X_i$, one has
    \[
        (\tilde{X}_i-X_i)\cdot \frac{d}{dt}(\tilde{X}_i(t)-X_i(t)) \le
        \frac{C}{N}|\tilde{X}_i(t)-X_i(t)|^2+(\tilde{X}_i(t)-X_i(t))\cdot
        (I_1+I_2)
    \]
    where
    \begin{multline*}
        I_1 = \frac{1-1/N}{p-1}
        \Big(\sum_{j: j\in \mathcal{C}}(\nabla_y
        \mathcal{K}(\tilde{X}_i,\tilde{X}_j)
        -\mathcal{K}(\tilde{X}_i,\tilde{X}_j)\nabla V(\tilde{X}_j))\\
        -\sum_{j: j\in \mathcal{C}}(\nabla_y \mathcal{K}(X_i,X_j)
        -\mathcal{K}(X_i,X_j)\nabla V(X_j)) \Big),
    \end{multline*}
    \begin{multline*}
        I_2 = \frac{1-1/N}{p-1}\sum_{j: j\in \mathcal{C}}(\nabla_y
        \mathcal{K}(X_i,X_j)-\mathcal{K}(X_i, X_j)\nabla V(X_j)) \\
        -\frac{1}{N}\sum_{j: j\neq i}(\nabla_y \mathcal{K}(X_i,X_j)
        -\mathcal{K}(X_i, X_j)\nabla V(X_j)).
    \end{multline*}

    Hence, introducing
    \[
        u(t)=\mathbb{E}|X_i(t)-\tilde{X}_i(t)|^2
        =\mathbb{E}|X_1(t)-\tilde{X}_1(t)|^2,
    \]
    we have
    \[
        \frac{d}{dt}u\le \frac{C}{N}u(t)
        + \mathbb{E}(X_i-\tilde{X}_i)\cdot I_1
        +\mathbb{E}(X_i-\tilde{X}_i)\cdot I_2.
    \]

    Due to the smoothness of $K$ and $V$ on torus, we easily find
    \[
    |I_1|\le C\frac{1}{p-1}\sum_{j\in\mathcal{C},j\neq
    i}(|X_i-\tilde{X}_i|+|X_j-\tilde{X}_j|)
    =C|X_i-\tilde{X}_i|+C\frac{1}{p-1}\sum_{j\in\mathcal{C},j\neq
    i}|X_j-\tilde{X}_j|,
    \]
    where $C$ is independent of $N$. Note that $\mathcal{C}$ is not
    independent of $X_j$ for $t>t_{m-1}$, so to continue we must
    consider conditional expectation.  Let $\mathcal{F}_{m-1}$ be the
    $\sigma$-algebra generated by $X_i(\tau), \tilde{X}_i(\tau)$ for
    $\tau\le t_{m-1}$ (including the initial data drawn independently)
    and the random division of the batches at $t_{m-1}$. Then,
    \eqref{eq:twoprocess} directly implies almost surely it holds that
    \begin{gather}\label{eq:conditionalexpect}
    \mathbb{E}(|X_j(t)-X_j(t_{m-1})| | \mathcal{F}_{m-1})\le C\eta,
    ~\mathbb{E}(|\tilde{X}_j(t)-\tilde{X}_j(t_{m-1})| |
    \mathcal{F}_{m-1})\le C\eta.
    \end{gather}
    Thus, defining the error process
    \begin{gather}
        Y_i(t)=\tilde{X}_i(t)-X_i(t),
    \end{gather}
    we have
    \begin{gather}
        \mathbb{E}(|Y_i(t)-Y_i(t_{m-1})|)\le C\eta,
    \end{gather}
    yielding
    \begin{gather}\label{eq:incrementofsqrtu}
        |\sqrt{u}(t)-\sqrt{u}(t_{m-1})|\le C\eta.
    \end{gather}

    Note that
    \[
    \mathbb{E}\left(
    |X_i-\tilde{X}_i|\frac{1}{p-1}\sum_{j\in\mathcal{C},j\neq
    i}|X_j-\tilde{X}_j|\right) \le
    \sqrt{u}\left(\frac{1}{p-1}\mathbb{E}\sum_{j\in\mathcal{C},j\neq
    i}|X_j-\tilde{X}_j|^2\right)^{1/2}.
    \]
    The inside of the parenthesis can be estimated as
    \begin{multline*}
    \frac{1}{p-1}\mathbb{E}\sum_{j\in\mathcal{C},j\neq
    i}|X_j-\tilde{X}_j|^2
    =\frac{1}{p-1}\mathbb{E}\sum_{j\in\mathcal{C},j\neq
    i}|X_j(t_{m-1})-\tilde{X}_j(t_{m-1})|^2\\
    +\frac{1}{p-1}\mathbb{E}\left(\mathbb{E}(
    (|X_j-\tilde{X}_j|^2-|X_j(t_{m-1})-\tilde{X}_j(t_{m-1})|^2)
    |\mathcal{F}_{m-1})\right)
    \end{multline*}
    The first term on the right hand side then becomes $u(t_{m-1})$ by
    Lemma \ref{thm:consistency}.  By \eqref{eq:conditionalexpect}, it is
    clear that
    \[
    \mathbb{E}(
    (|X_j-\tilde{X}_j|^2-|X_j(t_{m-1})-\tilde{X}_j(t_{m-1})|^2)
    |\mathcal{F}_{m-1}) \le
    2|X_j(t_{m-1})-\tilde{X}_j(t_{m-1})|C\eta+C\eta^2.
    \]
    Hence,
    \[
    \mathbb{E}(X_i-\tilde{X}_i)\cdot I_1 \le Cu(t)
    +Cu(t_{m-1})+C\sqrt{u(t_{m-1})}\eta+C\eta^2.
    \]
    where $C$ is independent of $N$.
    Since $u(t_{m-1})\le Cu(t)+C\eta^2$ by \eqref{eq:incrementofsqrtu},
    then
    \[
    \mathbb{E}(X_i-\tilde{X}_i)\cdot I_1 \le Cu(t)+C\eta^2.
    \]

    Letting $Z=(X_1,\ldots, X_N)$, one sees easily that
    $I_2=\chi_i(Z(t))$.  Then, we find
    \[
    Y_i(t)\cdot  I_2(t)=(Y_i(t)-Y_i(t_{m-1}))\cdot
    \chi_i(Z(t))+Y_i(t_{m-1})\cdot \chi_i(Z(t))=J_1+J_2.
    \]
    In $J_2$, $Y_i(t_{m-1})$ is independent of the random batch division
    at $t_{m-1}$. Then, Lemma \ref{thm:consistency} tells us that
    \[
        \mathbb{E}J_2=0.
    \]
    
    Using \eqref{eq:twoprocess}, we have
    \begin{gather}\label{eq:increments}
    Y_i(t)-Y_i(t_{m-1})=-\int_{t_{m-1}}^t
    \chi_i(Z(s))\,ds+\int_{t_{m-1}}^t
    f_i(\tilde{Z}(s))-f_i(Z(s))\,ds.
    \end{gather}
    Since $\chi_i$ is bounded,
    \[
    \left| \mathbb{E}\int_{t_{m-1}}^t \chi_i(Z(s))\cdot\chi_i(Z(t))\,ds
    \right| \le C\eta,
    \]
    where $C$ is related to the infinity norm of the variance of
    $\chi_i(t)$. This is the main term in the local truncation error.
    Just as we did for $I_1$,
    \[
    |f_i(\tilde{Z}(s))-f_i(Z(s))|\le
    C\frac{1}{p-1}\sum_{j\in\mathcal{C},j\neq
    i}(|X_i-\tilde{X}_i|+|X_j-\tilde{X}_j|)
    =C|X_i-\tilde{X}_i|+\frac{C}{p-1}\sum_{j\in\mathcal{C}, j\neq
    i}|X_j-\tilde{X}_j|.
    \]
    Since
    \begin{multline*}
    \mathbb{E}\frac{1}{p-1}\sum_{j\in\mathcal{C}, j\neq
    i}|X_j-\tilde{X}_j| \le
    \mathbb{E}\frac{1}{p-1}\sum_{j\in\mathcal{C}, j\neq
    i}|X_j(t_{m-1})-\tilde{X}_j(t_{m-1})| \\
    +\mathbb{E}\left(\frac{1}{p-1}\sum_{j\in\mathcal{C}, j\neq
    i}\mathbb{E}\left(|X_j(s)-\tilde{X}_j(s)
    -(X_j(t_{m-1})-\tilde{X}_j(t_{m-1}))|
    \Big|\mathcal{F}_{m-1}\right)\right)
    \end{multline*}
    This is controlled by $C\sqrt{u(t_{m-1})}+C\eta$. Hence, 
    \[
    \mathbb{E}J_1\le C\eta+C\sqrt{u(t_{m-1})}\eta+C\eta^2\le
    C\eta+Cu+C\eta^2,
    \]
    where the $\eta$ term is from the variance term.
    
    Eventually,
    \[
        \frac{d}{dt}u\le C(u+\eta+\eta^2)\le Cu+C\eta.
    \]
    Applying Gr\"onwall's inequality, we find
    \[
        \sup_{t\le T}u(t)\le C(T)\eta.
    \]
    The last claim for $W_2$ distance follows from the definition of $W_2$.
\end{proof}

Note that the one marginal $\mu_N^{(1)}$ is the distribution of $X_i$
for any $i$, which is deterministic. This should be distringuished from
the empirical measure $\mu_N=\frac{1}{N}\sum_i \delta(x-X_i(t))$ which
is random.  As can be seen from the proof,  the main contribution in the
local truncation error comes from the variance of the the noise
$\chi_i$. We believe the error bound here can be made independent of $T$
due to the intrinsic structure of SVGD discussed above in section
\ref{sec:svgd}. Often, such long time estimates are established by some
contracting properties, so one may want to find the intrinsic converging
structure of \eqref{eq:discreteODE}.  However, rigorously establishing
such results seems nontrivial due to the nonlocal effects of the
external forces (the $\nabla V$ terms).

\section{Numerical Experiments}\label{sec:experiment}

We consider some test examples in \cite{liu2017stein} to validate
RBM-SVGD algorithm and compare with the original SVGD algorithm. In
particular, in a toy example for 1D Gaussian mixture, RBM-SVGD is
proved to be effective in the sense that the particle system converges
to the expected distribution with less running time than the original
SVGD method. A more practical example, namely Bayesian logistic
regression, is also considered to verify the effectiveness of RBM-SVGD
on large datasets in high dimension. Competitive prediction accuracy is
presented by RBM-SVGD, and less time is needed. Hence, RBM-SVGD seems to
be a more efficient method.

All numerical results in this section are implemented with MATLAB R2018a
and performed on a machine with Intel Xeon CPU E5-1650 v2 @ 3.50GHz and
64GB memory.

\subsection{1D Gaussian Mixture}

As a first example, we use the Gaussian mixture probability in
\cite{liu2016stein} for RBM-SVGD. The initial distribution is
$\mathcal{N}(-10, 1)$, Gaussian with mean $-10$ and variance 1. The
target density is given by the following Gaussian mixture
\begin{align}\label{goalpdf}
    \pi(x) = \dfrac{1}{3} \cdot \dfrac{1}{\sqrt{2\pi}}
    e^{-(x+2)^2/2} + \dfrac{2}{3} \cdot \dfrac{1}{\sqrt{2\pi}}
    e^{-(x-2)^2/2}.
\end{align}
The kernel for the RKHS is the following Gaussian kernel
\begin{gather}\label{eq:Gaussiankernel}
    K(x)=\dfrac{1}{\sqrt{2\pi h}}e^{-x^2/2h},
\end{gather}
where $h$ is the bandwidth parameter. For a fair comparison with the
numerical results in \cite{liu2016stein}, we first reproduce their
results using $N=100$ particles and dynamic bandwidth parameter $h =
\frac{\mathrm{med}^2}{2\log N}$, where $\mathrm{med}$ is the median of
the pairwise distance between the current points. Since dynamic
bandwidth is infeasible for RBM-SVGD, we produce the results with fixed
bandwidth $h=2$ for the comparison between SVGD and RBM-SVGD.  The
RBM-SVGD uses Algorithm~\ref{ssvgd} with initial stepsize being 0.2 and
the following stepsizes are generated from AdaGrad.  Different batch
sizes are tested to demonstrate the efficiency of RBM-SVGD. Numerical
results are illustrated in Figure~\ref{fig:1dgm-comparison} with the
same initial random positions of particles following
$\mathcal{N}(-10,1)$ distribution.

\begin{figure}[htbp]
    \centering
    \includegraphics[width=\textwidth]{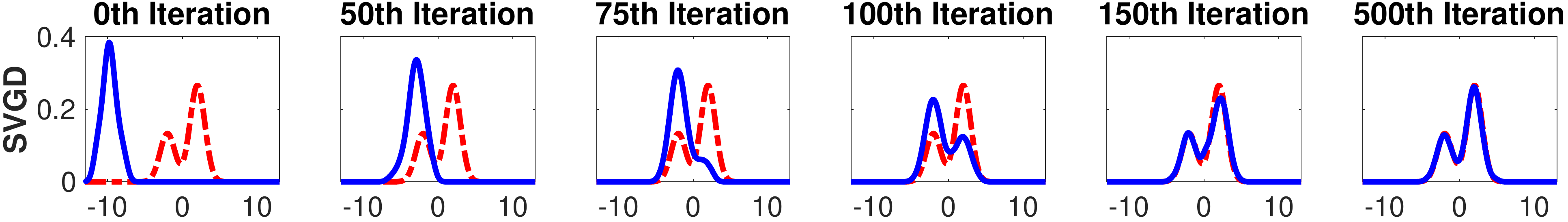}
    \includegraphics[width=\textwidth]{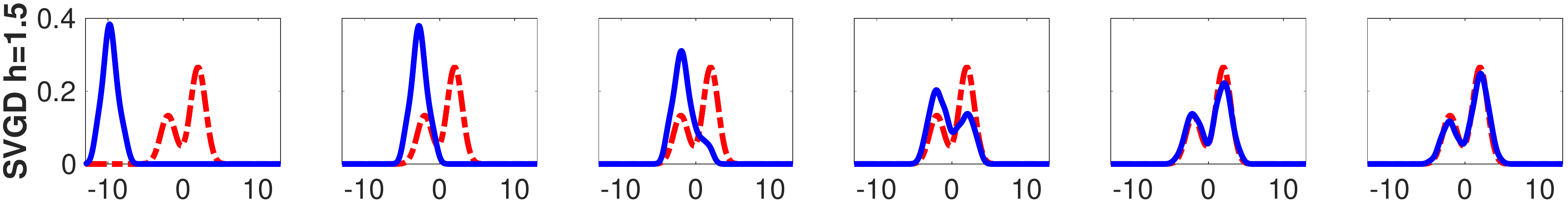}
    \includegraphics[width=\textwidth]{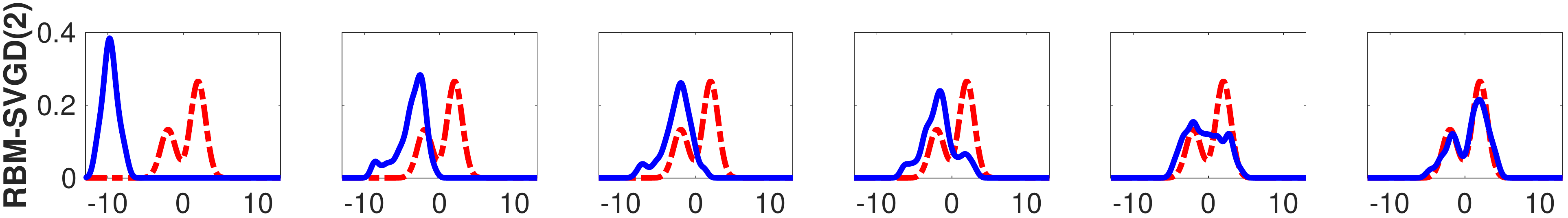}
    \includegraphics[width=\textwidth]{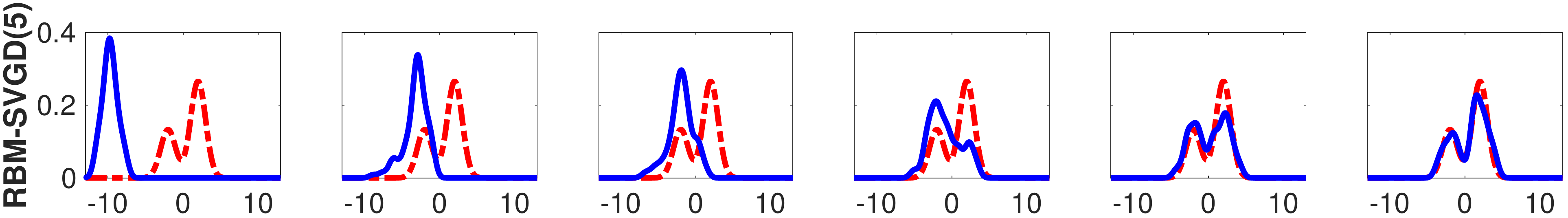}
    \includegraphics[width=\textwidth]{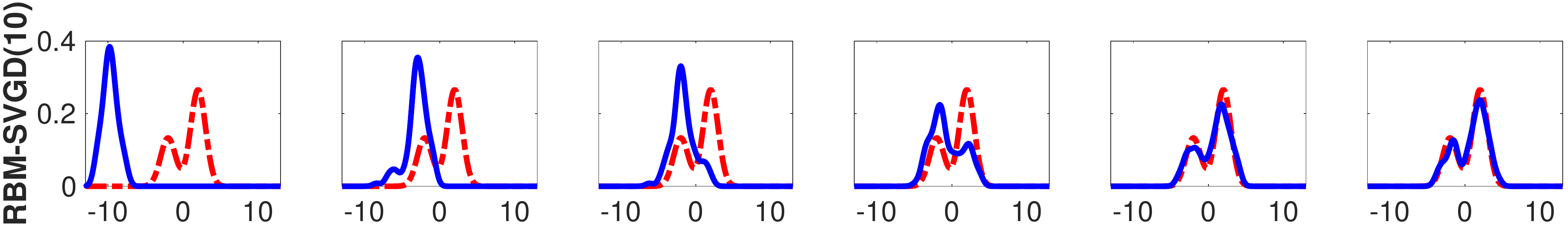}
    \includegraphics[width=\textwidth]{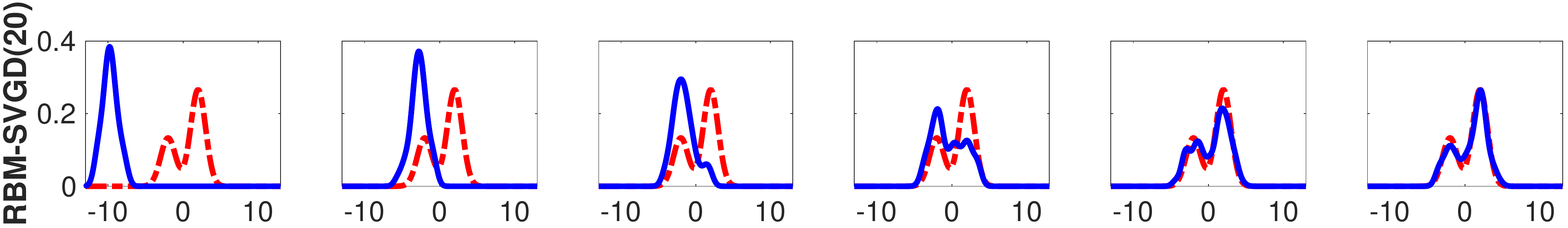}
    \caption{Comparison between SVGD and RBM-SVGD with different
    batch sizes using $N=100$ particles. The first row reproduces
    results in~\cite{liu2016stein}; the second row uses a fixed band
    width $h=2$ with other settings being the same as first row; the
    third to fifth rows apply RBM-SVGD with batch size 2, 5, and 20
    respectively and other settings are the same as the second row.
    In all figures, red dash curves indicate target density function
    whereas blue curves are empirical density estimators (estimated
    using kernel density estimator). }
    \label{fig:1dgm-comparison}
\end{figure}

As stated in \cite{liu2016stein}, the difficulty lies in the
strong disagreement between the initial density function and the
target density $\pi(x)$. According to the first and second row
in Figure~\ref{fig:1dgm-comparison}, SVGD with and without fixed
bandwidth parameter capture the target density efficiently and
the corresponding convergent behaviors are similar to each other.
Reading from the last column of Figure~\ref{fig:1dgm-comparison},
we observe that RBM-SVGD inherits the advantage of SVGD in the sense
that it can conquer the challenge and also show compelling result
with SVGD. When the batch size is small, e.g., $p=2$ or $p=5$, the
estimated densities differ from that of SVGD, and, according to our
experience, the estimated densities are not very stable across several
executions. While, in theory, RBM-SVGD runs $N/p$ times faster than
SVGD. Hence RBM-SVGD with $p=5$ at 500th iteration costs the same as
50 iterations of SVGD. According to Figure~\ref{fig:1dgm-comparison},
RBM-SVGD(2) at 500th iteration significantly outperform the 50th
iteration of SVGD.  As we increase the batch size, as the last two
rows of Figure~\ref{fig:1dgm-comparison}, more stable and similar
behavior as SVGD is observed.

Provided the good performance of RBM-SVGD, we also check the sampling
power and its computational cost. We conduct the following simulations
with $N=256$ particles for 500 iterations with the Gaussian kernel
\eqref{eq:Gaussiankernel}. For RBM-SVGD, we use fixed bandwidth
$h=0.35$ whereas SVGD use the aforementioned dynamic bandwidth
strategy. When we apply SVGD or RBM-SVGD with different batch sizes,
the same initial random positions of particles is used.  For a
given test function $h(x)$, we compute the estimated expectation
$\bar{h}=\frac{1}{N}\sum_{i=1}^N h(X_i(T))$ and the sampling accuracy
is measured via the Minimum Square Error (MSE) over $100$ random
initializations following the same distribution as before:
\[
    \text{MSE} = \frac{1}{100} \sum_{j=1}^{100}( \bar{h}_j
    - \mathbb{E}_{X\sim \pi}h(X))^2,
\]
where $\mathbb{E}_{X \sim \pi}h(X)$ denotes the underlying truth.
Three test functions are explored, $h_1(x)=x$, $h_2(x) = x^2$, and
$h_3(x) = \cos 2x$, with their corresponding true expectations being
$\frac{2}{3}$, $5$, and $\frac{\cos 4}{e^2}$. The reported runtime is
also averaged over $100$ random initializations.

\begin{figure}[htp]
    \centering
    \begin{subfigure}[b]{0.32\textwidth}
        \includegraphics[height=0.8\textwidth]{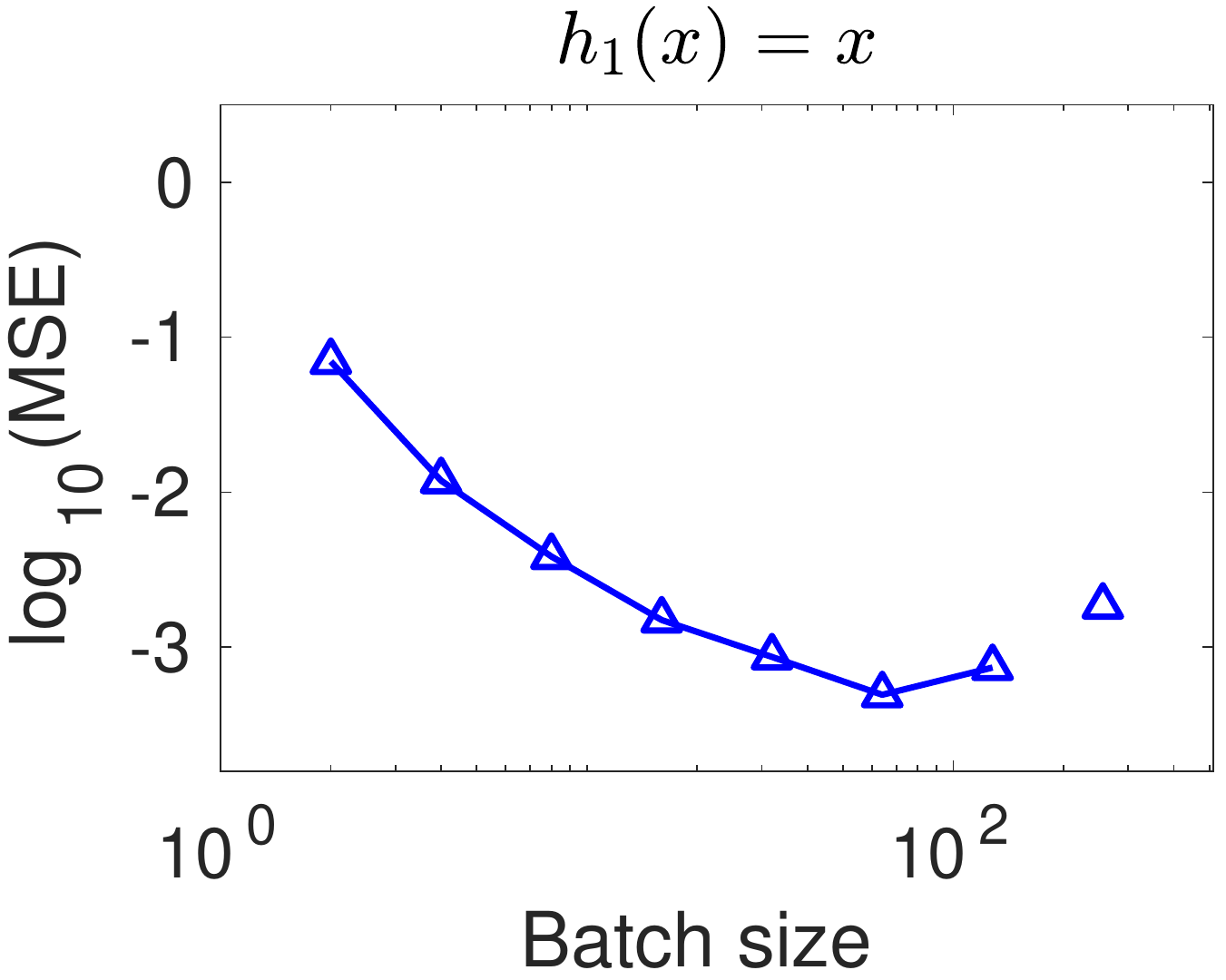}
        \caption{}
        \label{fig:h1}
    \end{subfigure}
    \hspace{.2cm}
    \begin{subfigure}[b]{0.32\textwidth}
        \includegraphics[height=0.8\textwidth]{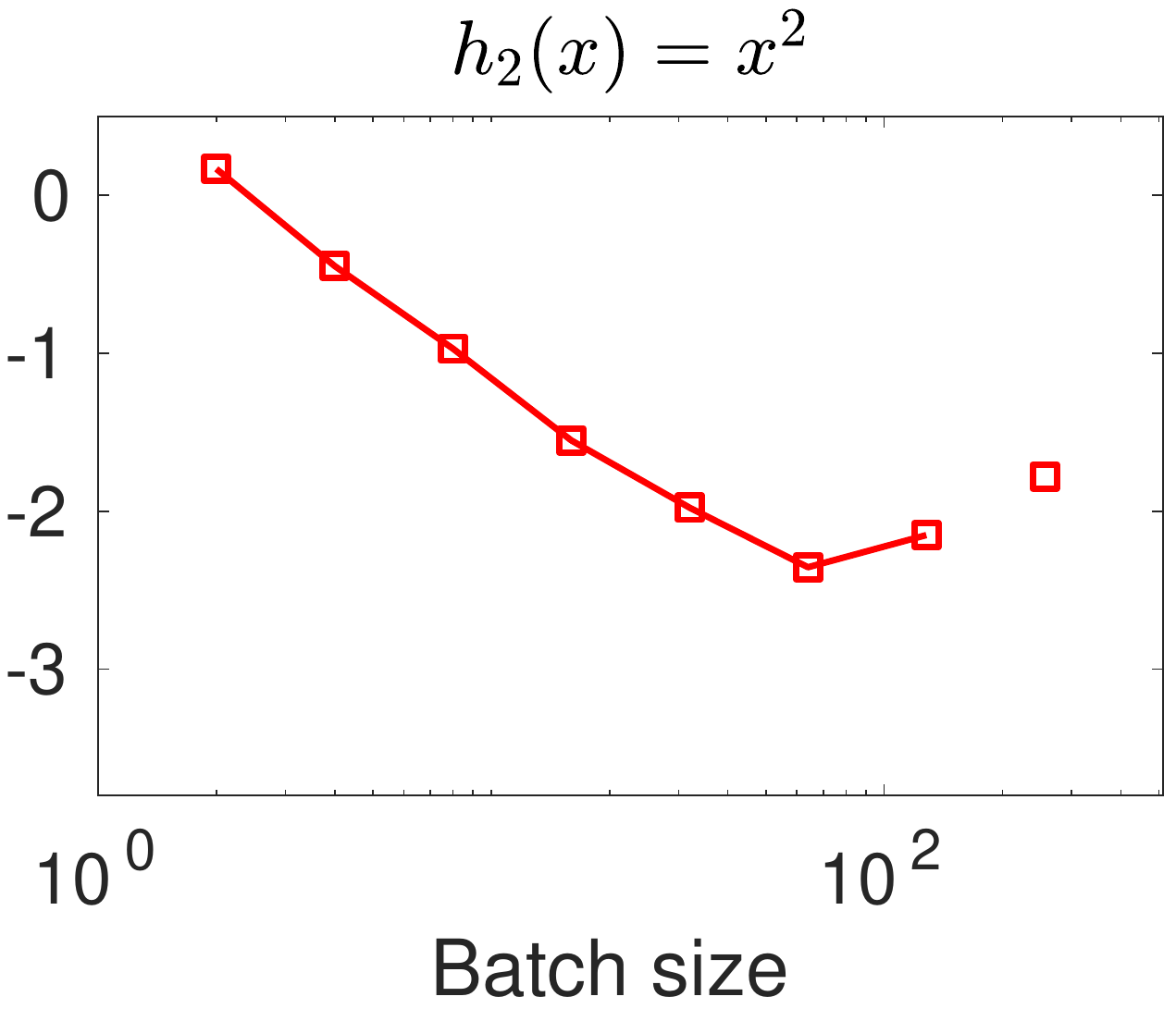}
        \caption{}
        \label{fig:h2}
    \end{subfigure}
    ~
    \begin{subfigure}[b]{0.32\textwidth}
        \includegraphics[height=0.8\textwidth]{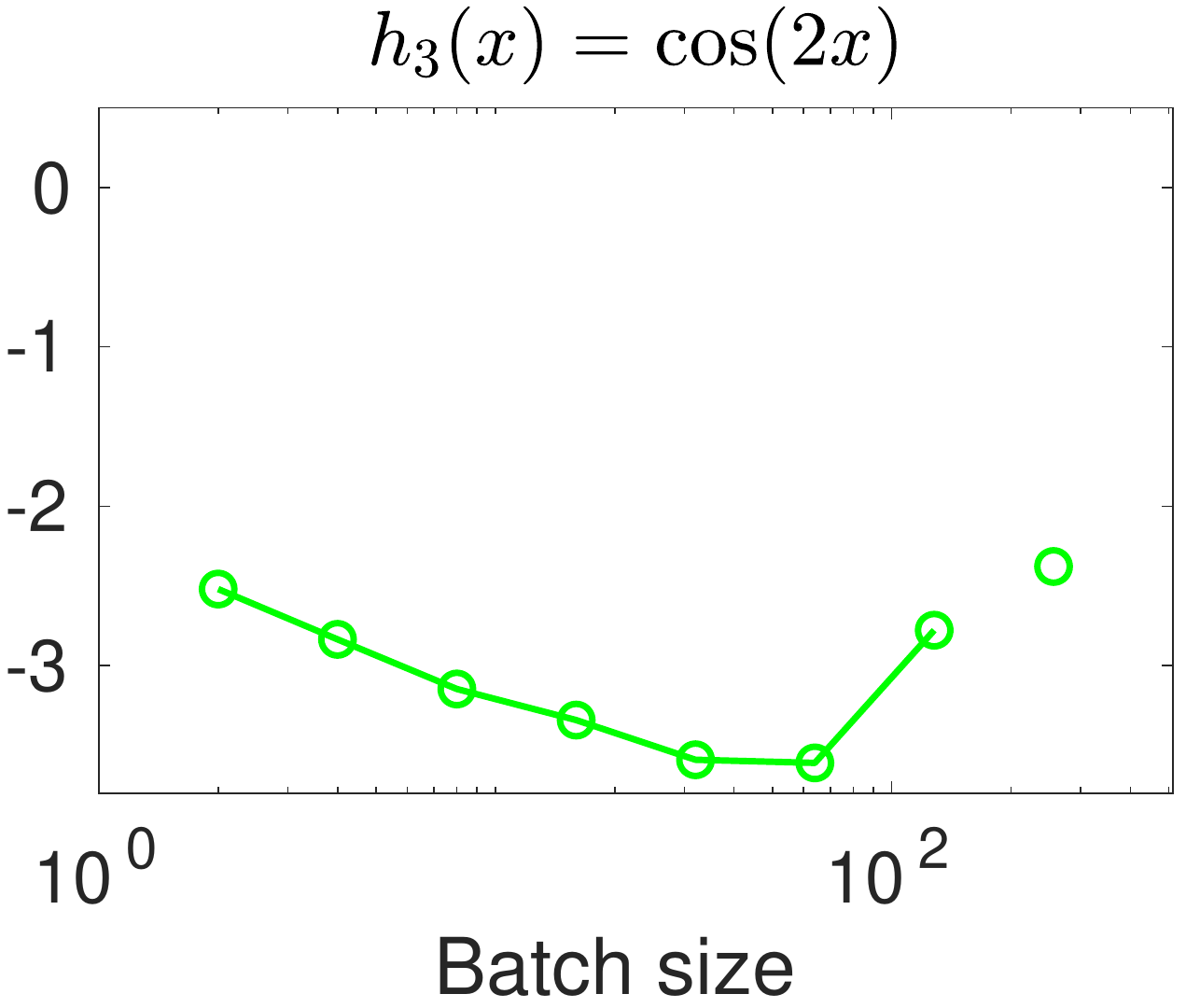}
        \caption{}
        \label{fig:h3}
    \end{subfigure}
    \caption{MSEs of (a) $h_1(x)=x$, (b) $h_2(x)=x^2$, and
    (c) $h_3(x)=\cos 2x$, against different batch sizes.}
    \label{fig:1dgm-estimating}
\end{figure}

\begin{table}[htp]
    \centering
    \caption{Averaged runtime for different batch sizes.}
    \label{tab:1dgm-runtime}
    \begin{tabular}{ccccccccc}
        \toprule
        & \multicolumn{7}{c}{RBM-SVGD} & SVGD \\
        \cmidrule{2-8}
        Batch size & 2 & 4 & 8 & 16 & 32 & 64 & 128 & 256\\
        \midrule
        Runtime(s) & 0.055 & 0.095 & 0.178 & 0.341 & 0.270 & 0.238 &
        0.314 & 0.733 \\
        Speedup & 13.3x & 7.7x & 4.1x & 2.1x & 2.7x & 3.1x &
        2.3x &  \\
        \bottomrule
    \end{tabular}
\end{table}

Figure~\ref{fig:1dgm-estimating} (a), (b), and (c) show the MSE
against different batch sizes for $h_1(x)$, $h_2(x)$, and $h_3(x)$
respectively. The results of RBM-SVGD with different batch sizes are
connected through lines, whereas the results of SVGD are the isolated
points with batch size $p=256$. In general, the estimations of $h_1(x)$
and $h_2(x)$ are better than that of $h_3(x)$, which agrees with the
difficulty of the problems. However, in all three figures, we observe
that the MSE decays first as $p$ increases and then increases for $p
\geq 64$. Such a behavior is due to the choice of bandwidth parameter.
Table~\ref{tab:1dgm-runtime} shows the averaged runtime of RBM-SVGD
and SVGD for different batch sizes. RBM-SVGD is faster than SVGD for
all choices of batch sizes. Ideally, RBM-SVGD with $p=2$ should be 128
times faster than SVGD, which turns out to be $13.3$ times speedup in
runtime. This is due to the nature of Matlab, since Matlab is better
optimized for block matrix operations. We expect that if the code
is implemented with other programming languages, e.g., C++, Fortran,
etc., close-to-optimal speedup should be observed.

\subsection{Bayesian Logistic Regression}\label{sec:regression}

In this experiment, we apply RBM-SVGD to conduct
Bayesian logistic regression for binary classification
for the Covertype dataset with 581012 data points and
54 features~\cite{gershman2012nonparametric}.  Under the same
setting as Gershman~\cite{gershman2012nonparametric,liu2016stein},
the regression weights $w$ are assigned with a Gaussian prior
$p_0(\omega|\alpha)=\mathcal{N}(w,\alpha^{-1})$, and the variance
satisfies $p_0(\alpha)=\Gamma(\alpha, 1,0.01)$, where $\Gamma$
represents the density of Gamma distribution. The inference is applied
on posterior $p(x|D)$ with $x=[w,\log\alpha]$. The kernel $K(\cdot)$ is
taken again to be the same Gaussian kernel as \eqref{eq:Gaussiankernel}.

\begin{figure}[htp]
    \centering
    \includegraphics[width=0.7\textwidth]{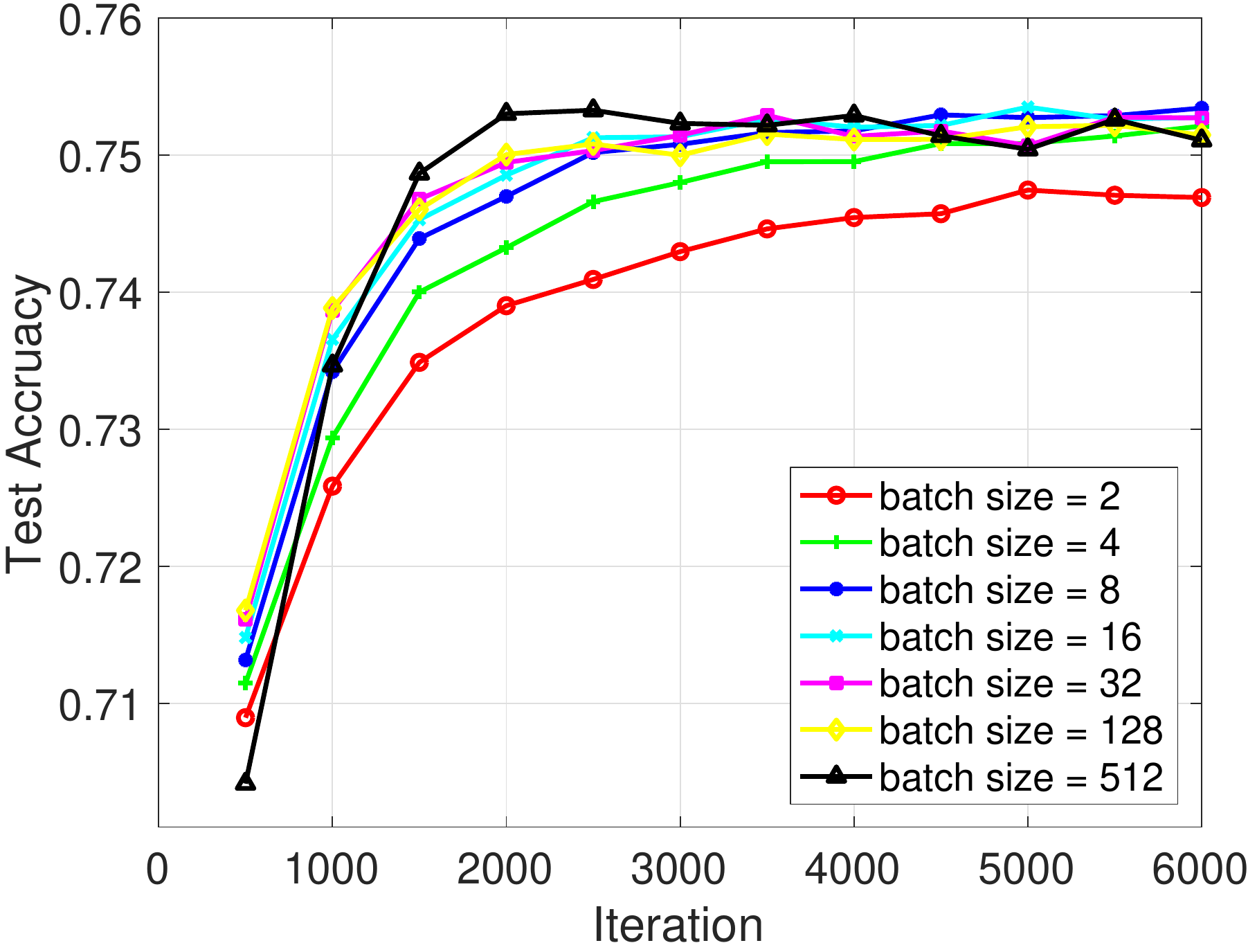}
    \caption{Test accuracy under different batch sizes of RBM-SVGD.}
    \label{fig:bayselr-comparison}
\end{figure}

Since the problem is in high dimension, we adopt $N=512$ particles
$N=512$ in this experiment, which also create more space for
the selection of batch sizes. The training is done on 80\% of
the dataset, and the other 20\% is used as the test dataset.
For particle system \eqref{eq:discreteODE}, the computation of
$-\nabla V=\nabla \log p(x)$ is expensive. Hence, we use the
same strategy as mentioned in \cite[section 3.2]{liu2016stein},
i.e. using data-mini-batch~\footnote{To avoid confusion with
our batch of particles, we call it data-mini-batch instead.}
of the data to form a stochastic approximation of $p(x)$ with the
data-mini-batch size being $100$. Since $\nabla \log p$ depends only
on $x$ as in Algorithm~\ref{ssvgd}, at each time step, we call this
function only once and compute $\nabla\log p$ for all particles,
which means the same mini-batches are used for $\nabla \log p$ of
all particles.  In this experiment, we use fixed bandwidth $h = 256$
for RBM-SVGD and dynamic bandwidth strategy for SVGD. The RBM-SVGD
uses Algorithm~\ref{ssvgd} with initial stepsize being 0.05 and the
following stepsizes are generated from AdaGrad.  Different batch
sizes are tested to demonstrate the efficiency of RBM-SVGD. Each
configuration is executed on $50$ random initializations. The
averaged test accuracies for different batch sizes are illustrated
in Figure~\ref{fig:bayselr-comparison}.

\begin{table}[htp]
    \centering
    \caption{Average runtime of $6000$ iterations}
    \label{tab:bayselr-runtime}
    \begin{tabular}{cccccccc}
        \toprule
        & \multicolumn{6}{c}{RBM-SVGD} & SVGD \\
        \cmidrule{2-7}
        Batch size & 2 & 4 & 8 & 16 & 32 & 128 & 512\\
        \midrule
        Runtime(s) & 8.59 & 11.24 & 16.28 & 26.15 & 21.66 & 19.42 &
        47.01 \\
        Speedup & 5.5x & 4.2x & 2.9x & 1.8x & 2.2x & 2.4x & \\
        \bottomrule
    \end{tabular}
\end{table}

\begin{table}[htp]
    \centering
    \caption{Statistics of RBM-SVGD and SVGD.}
    \label{tab:bayselr-stability}
    \begin{tabular}{cccccccc}
        \toprule
        & Iteration & 1000 & 2000 & 3000 & 4000 & 5000 & 6000\\
        \midrule
        \multirow{2}{*}{\makecell{RBM-SVGD\\$p=2$}} & Mean &
        0.7090 & 0.7349 & 0.7409 & 0.7446 & 0.7457 & 0.7471 \\
        \cmidrule{2-8}
        & Std & 0.0045 & 0.0040 & 0.0040 & 0.0034 & 0.0034 & 0.0038\\
        \midrule
        \multirow{2}{*}{\makecell{RBM-SVGD\\$p=8$}} & Mean & 0.7342 &
        0.7470 & 0.7508 & 0.7518 & 0.7527 & 0.7534 \\
        \cmidrule{2-8}
        & Std & 0.0073 & 0.0056 & 0.0041 & 0.0045 & 0.0039 & 0.0033 \\
        \midrule
        \multirow{2}{*}{SVGD} & Mean & 0.7347 & 0.7530 & 0.7523 & 0.7529
        & 0.7504 & 0.7511 \\
        \cmidrule{2-8}
        & Std & 0.0068 & 0.0048 & 0.0071 & 0.0048 & 0.0061 & 0.0062 \\
        \bottomrule
    \end{tabular}
\end{table}

As shown in Figure~\ref{fig:bayselr-comparison}, RBM-SVGD is almost as
efficient as SVGD even for small batch sizes. When $p=2$, the test
accuracy converges to a value slightly off that of SVGD. RBM-SVGD with
$p=4$ converges to the same accuracy as SVGD but at a slower convergent
rate. RBM-SVGD with batch size greater than $4$, we observe similar
convergent behavior as that of SVGD. The runtime of RBM-SVGD, as shown
in Table~\ref{tab:bayselr-runtime}, is faster than SVGD, where the
runtime of $6000$ iterations is reported. Comparing to the similar
runtime table for 1D Gaussian mixture example, as
Table~\ref{tab:1dgm-runtime}, the acceleration of RBM-SVGD is not as
significant as before. This is due to the linear but expensive
evaluation of $\nabla \log p$, where RBM-SVGD and SVGD spend the same
amount time in the evaluation each iteration. Although the evaluation of
$\nabla \log p$ is expensive, it is linear in $N$. As $N$ increases, the
advantage of RBM-SVGD would be more significant. In
Table~\ref{tab:bayselr-stability}, we list the mean and standard
deviation of RBM-SVGD with $p=2$, $p=8$, and SVGD of different
iterations. Based on the statistics, we conclude that RBM-SVGD and SVGD
are of similar prediction power and RBM-SVGD is efficient also in
high-dimensional particle systems as well.

\section{Conclusion}

We have applied the random batch method for interacting particle systems
to SVGD, resulting in RBM-SVGD, which turns out to be a cheap sampling
algorithm and inherits the efficiency of the original SVGD algorithm.
Theory and Numerical experiments have validated the algorithm and hence,
it can potentially have many applications, like Bayesian inference.
Moreover, as a hybrid strategy, one may
increase the batch size as time goes on to increase the accuracy, or apply some variance reduction approach.

\section*{Acknowledgement} This work is supported by KI-Net
NSF RNMS11-07444.  The work of L. Li was partially sponsored by Shanghai Sailing Program 19YF1421300, the work of Y. Li was partially supported by
OAC-1450280, the work of J.-G. Liu was partially supported by NSF
DMS-1812573, and the work of J. Lu was supported in part by NSF
DMS-1454939.

\bibliographystyle{unsrt} \bibliography{sdealg}

\appendix

\section{Proof of Lemma \ref{thm:consistency}}\label{app:proofconsistency}

\begin{proof}[Proof of Lemma \ref{thm:consistency}]
    The proof is pretty like the one in \cite{jin2018random}. We use the
    random variable $I(i, j)$ to indicate whether $i$ and $j$ are in a
    common batch. In particular, $I(i,j)=1$ if $i$ and $j$ are in a
    common batch while $I(i,j)=0$ if otherwise.  Then, it is not hard to
    compute (see \cite{jin2018random})
    \begin{gather}\label{eq:Indicatorexp}
        \begin{split}
            & \mathbb{E}1_{I(i, j)=1}=\frac{p-1}{N-1},\\
            & \mathbb{P}(I(i,j)I(j,k)=1)=\frac{(p-1)(p-2)}{(N-1)(N-2)}.
        \end{split}
    \end{gather}

    We note
    \begin{gather}
        \chi_i(x)=\frac{1}{N}\sum_{j: j\neq i}
        \left(1-\frac{N-1}{p-1}I(i,j)\right)F(x_i, x_j).
    \end{gather}
    The first equation in \eqref{eq:Indicatorexp} clearly implies that
    $\mathbb{E}\chi_i(x)=0$. Using \eqref{eq:Indicatorexp}, we can
    compute directly that
    \begin{multline*}
        \mathbb{E}|\chi_i(x)|^2 = \frac{1}{N^2}
        \Big( \sum_{j: j\neq i}(\frac{N-1}{p-1}-1)|F(x_i, x_j)|^2 \\
        +\sum_{j,k: j\neq i, k\neq i, j\neq k}
        \left(\frac{(N-1)(p-2)}{(N-2)(p-1)}-1\right)
        F(x_i, x_k)\cdot F(x_i, x_j) \Big)
    \end{multline*}
    Rearranging this, we get the claimed expression.
\end{proof}

\end{document}